\documentclass[letterpaper, 10 pt, conference]{ieeeconf}  

\IEEEoverridecommandlockouts                              
\usepackage{amsmath} 
\usepackage{amssymb}  
\usepackage{subcaption}

\usepackage{mathtools}
\newtheorem{theorem}{\bfseries Theorem}[section]
\newtheorem{lemma}{\bfseries Lemma}[section]
\newtheorem{assumption}{\bfseries Assumption}[section]
\newtheorem{definition}{\bfseries Definition}[section]
\newtheorem{proposition}{\bfseries Proposition}[section]
\usepackage{algorithm}
\usepackage{algpseudocode}
\usepackage{xcolor}
\usepackage{hyperref}
\hypersetup{breaklinks=true,colorlinks=true,linkcolor=black,citecolor=black}
\usepackage[noabbrev,capitalise,nameinlink]{cleveref}
\usepackage[sort]{cite}
\crefname{assumption}{Assumption}{Assumptions}
\allowdisplaybreaks

\title{\LARGE \bf
Finite-Time Analysis of Q-Value Iteration for General-Sum Stackelberg Games
}

\author{Narim Jeong and Donghwan Lee
\thanks{N. Jeong and D. Lee are with the Department of Electrical Engineering, Korea Advanced Institute of Science and Technology (KAIST), Daejeon, 34141, South Korea
        {\tt\small \{nrjeong, donghwan\}@kaist.ac.kr}}
}

\begin{document}

\maketitle
\thispagestyle{empty}
\pagestyle{empty}

\begin{abstract}
Reinforcement learning has been successful both empirically and theoretically in single-agent settings, but extending these results to multi-agent reinforcement learning in general-sum Markov games remains challenging. This paper studies the convergence of Stackelberg Q-value iteration in two-player general-sum Markov games from a control-theoretic perspective. We introduce a relaxed policy condition tailored to the Stackelberg setting and model the learning dynamics as a switching system. By constructing upper and lower comparison systems, we establish finite-time error bounds for the Q-functions and characterize their convergence properties. Our results provide a novel control-theoretic perspective on Stackelberg learning. Moreover, to the best of the authors' knowledge, this paper offers the first finite-time convergence guarantees for Q-value iteration in general-sum Markov games under Stackelberg interactions.
\end{abstract}

\section{INTRODUCTION}
Reinforcement learning (RL) has been successfully applied to a wide range of sequential decision-making problems and has demonstrated remarkable empirical performance in complex domains. Among the many RL algorithms, Q-learning~\cite{watkins1992q} has been one of the most fundamental and widely studied methods, with well-established convergence guarantees in single-agent Markov decision processes~\cite{fan2020theoretical,li2020sample,chen2021lyapunov,lee2020periodic}. These theoretical foundations, together with strong empirical success, have contributed to the widespread adoption of RL in practice.

In many real-world applications, however, decision-making involves multiple interacting agents whose outcomes depend on one another’s actions. This motivates the study of multi-agent reinforcement learning (MARL)~\cite{tan1993multi}, which extends classical RL to multi-agent environments. A common modeling approach in general-sum Markov games assumes that agents aim to reach a Nash equilibrium at each state, leading to algorithms such as Nash Q-learning~\cite{hu2003nash}. While this framework provides a natural extension of minimax Q-learning~\cite{littman1994markov} and captures a wide range of mixed cooperative–competitive interactions, it also introduces significant analytical challenges. In particular, it is known that establishing convergence for MARL in general-sum Markov games is notoriously difficult. There are equilibrium selection issues and non-ergodic learning dynamics~\cite{greenwald2003correlated} since the Nash equilibrium operator is not a contraction~\cite{hu2003nash} in general, and multiple equilibria may exist at a given state. As a result, convergence guarantees for Nash-based MARL are typically limited to restrictive settings or weaker solution concepts such as correlated equilibria~\cite{littman2001value,greenwald2003correlated,cisneros2023finite,song2021can,erez2023regret}. Moreover, even value iteration in general-sum Markov games may fail to converge and instead exhibit cyclic behaviors~\cite{zinkevich2005cyclic,giannou2022convergence}. Consequently, a general and intuitive convergence theory for MARL under Nash equilibria remains largely underdeveloped.

In contrast to Nash equilibria, many real-world multi-agent systems exhibit inherently asymmetric or hierarchical interactions, where one agent acts as the leader and others respond accordingly. Such scenarios arise naturally in applications such as autonomous driving~\cite{sadigh2016planning,wang2021game}, auction design~\cite{krishna2009auction}, and security games~\cite{sinha2018stackelberg}. In these settings, the Stackelberg equilibrium~\cite{bacsar1998dynamic} provides a more appropriate solution concept, as it explicitly models sequential decision-making, with the follower observing the leader's action and responding optimally. Such scenarios have motivated recent interest in Stackelberg learning in both static and dynamic games~\cite{fiez2020implicit,fiez2019convergence,vu2022stackelberg}. Despite its practical relevance, analyzing Stackelberg learning in Markov games is even more challenging than the Nash case~\cite{beck2023survey}. The induced Bellman operator becomes inherently asymmetric and policy-dependent because the follower's best response depends on the leader's action, which in turn depends on the value function. This creates a nested optimization structure and leads to highly nonlinear and potentially non-contractive dynamics. As a result, existing theoretical results for Stackelberg learning are limited to local convergence~\cite{fiez2020implicit,fiez2019convergence,zheng2022stackelberg,he2025learning} and rely on strong assumptions, such as a myopic follower or centralized coordination~\cite{zhong2021can,yu2022learning,niu2025finding}.

In this paper, we study the convergence properties of Stackelberg Q-value iteration in two-player general-sum Markov games. Our approach departs from equilibrium-based analyses and instead focuses on the evolution of Q-functions under the induced learning dynamics. By introducing a relaxed policy condition tailored to the Stackelberg setting, we provide a more intuitive and structured understanding of the learning process. In particular, we model the Stackelberg Q-value iteration as a switching system~\cite{liberzon2003switching}, where the switching behavior naturally captures the evolution of Q-function-induced dynamics over time. By constructing upper and lower comparison systems, we establish finite-time error bounds for the Q-functions and characterize their convergence behavior.

The main contributions of this paper can be summarized as follows: (i) We introduce a relaxed policy condition for the Stackelberg setting by adapting the worst-case response assumption commonly used in Nash Q-learning analyses~\cite{bowling2000convergence}. As this minimax-based assumption is overly restrictive and does not directly apply to asymmetric structure, we replace it with an $\epsilon$-relaxation on the upper bound. We show that this condition is valid and admits a natural interpretation in terms of best-response behavior. (ii) We establish finite-time error bounds for Stackelberg Q-value iteration in two-player general-sum Markov games. Existing works primarily provided asymptotic or local convergence guarantees~\cite{fiez2020implicit,fiez2019convergence,zheng2022stackelberg,he2025learning} or relied on restrictive assumptions~\cite{zhong2021can,yu2022learning,niu2025finding} in Stackelberg learning in general-sum games. While~\cite{bai2021sample} presented finite-time analysis in Stackelberg general sum settings, it did not extend to Markov games or Q-value iteration. In contrast, our analysis establishes an explicit finite-time convergence result under a relaxed policy condition. To the best of the authors' knowledge, this is the first work that provides finite-time convergence guarantees for Q-value iteration in general-sum Markov games under Stackelberg interactions. (iii) We develop a novel analytical framework based on switching systems to characterize the evolution of Stackelberg Q-functions. This framework captures the time-varying structure of Stackelberg-type interactions and enables a systematic comparison-based analysis. Our results extend recent switching-system-based analyses beyond the single-agent and zero-sum settings~\cite{lee2020unified,lee2021discrete,jeong2024finite,jeong2025finite}.

\section{PRELIMINARIES AND PROBLEM FORMULATION}
\textbf{Stackelberg Markov games and learning dynamics.}
We consider a two-player general-sum Markov game~\cite{shapley1953stochastic}, in which the two agents are referred to as leader and follower. The state space is defined as ${\cal S} := \{1,2,\ldots,|{\cal S}|\}$, and the action spaces of the leader and the follower are given by ${\cal A} := \{1,2,\ldots,|{\cal A}|\}$ and ${\cal B} := \{1,2,\ldots,|{\cal B}|\}$, respectively, where $|{\cal S}|$, $|{\cal A}|$, and $|{\cal B}|$ denote the cardinalities of the corresponding spaces. In the Stackelberg setting~\cite{fiez2019convergence,he2025learning,zhu2023multi}, the leader first selects an action $a \in {\cal A}$ and then the follower observes $a$ and selects $b \in {\cal B}$. The policies of the leader and the follower are defined as $\pi^1(a \mid s)$ and $\pi^2(b \mid s,a)$, respectively, where the policy of the follower explicitly depends on the leader's action. Then, the state $s$ transitions to the next state $s'$ according to the transition probability $P(s' \mid s,a,b)$. Each player $i \in \{1,2\}$ receives reward $r^i(s,a,b)$, which we denote by $r^i$.

For a policy pair $(\pi^1,\pi^2)$, the value function of player $i$ is defined as 
\[ V^i_{\pi^1,\pi^2}(s) := \mathbb{E} \left[ \sum_{k=0}^{\infty} \gamma^k r^i_{k+1} \middle| s_0=s \right], \quad i \in \{1,2\}, \]
where $r^i_k := r^i(s_k,a_k,b_k)$ denotes the reward received by player $i$ at time step $k$, and $\gamma \in [0,1)$ is the discount factor. The best response of the follower given $\pi^1$ is defined as
\begin{equation} \label{pi2_opt}
    \pi^{2*}(\pi^1) \in \arg\max_{\pi^2} V^2_{\pi^1,\pi^2}(s).
\end{equation}
Based on this best response, the leader chooses a policy that maximizes its own value, i.e.,
\begin{equation} \label{pi1_opt}
    \pi^{1*} \in \arg\max_{\pi^1} V^1_{\pi^1,\pi^{2*}(\pi^1)}(s).
\end{equation}
Then, a policy pair $(\pi^{1*},\pi^{2*})$ is said to be a Stackelberg equilibrium if it satisfies~\eqref{pi2_opt} and~\eqref{pi1_opt}. This leads to the following bilevel optimization formulation:
\[ \max_{\pi^1} V^1_{\pi^1,\pi^{2*}(\pi^1)}(s), \text{ where } \pi^{2*}(\pi^1) = \arg\max_{\pi^2} V^2_{\pi^1,\pi^2}(s). \]

In this paper, we assume that the policy is deterministic as $a_k(s)$ and $b_k(s,a)$ at iteration $k$, and the $\arg\max$ operator returns a unique maximizer. At each state $s$, the best response of the follower to the leader's action is defined as
\[ b_k(s,a) \in \arg\max_{b \in {\cal B}} Q^2_k(s,a,b), \]
where $Q^i_k$ represents the Q-function estimate for the leader and the follower at iteration $k$, respectively. Given this, the leader selects its action according to
\begin{equation} \label{pi1}
    a_k(s) \in \arg\max_{a \in {\cal A}} Q^1_k \left(s,a,b_k(s,a)\right).
\end{equation}
Then, the Q-functions for the leader and the follower are updated using the following recursion:
\begin{equation} \label{p1 q update}
\begin{aligned}
    Q^1_{k+1}&(s,a,b) = r^1(s,a,b) \\
    &+ \gamma \sum_{s' \in {\cal S}} P(s' \mid s,a,b) \max_{a \in {\cal A}} Q^1_k\left(s',a,b_k(s',a)\right)
\end{aligned}
\end{equation}
and
\begin{equation} \label{p2 q update}
\begin{aligned}
    Q^2_{k+1}&(s,a,b) = r^2(s,a,b) \\
    &+ \gamma \sum_{s' \in {\cal S}} P(s' \mid s,a,b) \max_{b \in {\cal B}} Q^2_k(s',a_k(s'),b).
\end{aligned}
\end{equation}
When a policy pair $(a_*,b_*)$ satisfies
\[ b_*(s,a) \in \arg\max_{b \in {\cal B}} Q^2_*(s,a,b) \]
and
\[ a_*(s) \in \arg\max_{a \in {\cal A}} Q^1_* \left(s,a,b_*(s,a)\right), \]
the pair $(a_*,b_*)$ constitutes a Stackelberg equilibrium.

Given the induced learning dynamics on the Q-functions, the iterates may converge to a fixed point under suitable conditions~\cite{fiez2019convergence,he2025learning}. However, in the absence of such conditions, the updates can exhibit non-convergent behaviors such as limit cycles. In such cases, the resulting policies form a periodic orbit rather than a stationary Stackelberg equilibrium.

\vspace{0.3cm}
\textbf{Switching system. }
A switching system~\cite{liberzon2003switching} is a class of nonlinear systems~\cite{khalil2002nonlinear} that operates among multiple subsystems according to a switching signal. Regarding the various types of switching systems, this paper focuses on an \textit{affine switching system} defined as
\[ x_{k+1} = A_{\sigma_k}x_k + b_{\sigma_k}, \]
where $x_k \in \mathbb{R}^n$ denotes the system state at time step $k$, $\sigma_k \in {\mathcal M} := {1,2,\ldots,{\mathcal M}}$ denotes the switching signal, and $A_{\sigma_k} \in \mathbb{R}^{n \times n}$ and $b_{\sigma_k} \in \mathbb{R}^n$ are the active subsystems. Both $A_{\sigma_k}$ and $b_{\sigma_k}$ depend on $\sigma_k$, which can be chosen either arbitrarily or according to a policy. It is known that the presence of the affine term $b_{\sigma_k}$ can make stability analysis more challenging.

\vspace{0.3cm}
\textbf{Problem formulation. }
Throughout the paper, we will use the following assumptions and notations to simplify the analysis.

\vspace{0.3cm}
\begin{assumption} \label{assumption1}
\leavevmode\\[-1em]
\begin{enumerate}
    \item The reward is bounded within a unit interval: $\max_{(s,a,b) \in {\cal S} \times {\cal A} \times {\cal B}} |r (s,a,b)| \leq 1$.
    \item The initial values of the Q-functions are bounded within a unit interval: $\left\|Q^1_0\right\|_\infty \le 1$ and $\left\|Q^2_0\right\|_\infty \le 1$.
\end{enumerate}
\end{assumption}

\vspace{0.3cm}
\begin{definition} \label{definition1}
\leavevmode\\[-1em]
\begin{enumerate}
    \item Q-function vector:
    \[ Q := \left[ {\begin{array}{*{20}c}
            Q_{1,1} \\
            \vdots \\
            Q_{|{\cal A}|,|{\cal B}|} \\
        \end{array}} \right],
    Q_{a,b} := \left[ {\begin{array}{*{20}c}
            Q(1,a,b) \\
            \vdots \\
            Q(|{\cal S}|,a,b) \\
        \end{array}} \right], \]
    where $Q \in \mathbb{R}^{|{\cal S}||{\cal A}||{\cal B}|}$ and $Q_{a,b} \in \mathbb{R}^{|{\cal S}|}$.

    \item State transition probability matrix:
    \[ P:= \left[ {\begin{array}{*{20}c}
            P_{1,1} \\
            \vdots \\
            P_{|{\cal A}|,|{\cal B}|}
        \end{array}} \right], \]
    \[ P_{a,b} := \left[ {\begin{array}{*{20}c}
            P(1 \mid 1,a,b) & \cdots & P(|{\cal S}| \mid 1,a,b) \\
            \vdots & \ddots & \vdots \\
            P(1 \mid |{\cal S}|,a,b) & \cdots & P(|{\cal S}| \mid |{\cal S}|,a,b)
        \end{array}} \right], \]
    where $P \in \mathbb{R}^{|{\cal S}||{\cal A}||{\cal B}| \times |{\cal S}|}$ and $P_{a,b} \in \mathbb{R}^{|{\cal S}|\times |{\cal S}|}$.

    \item Stackelberg action-selection matrix:
    For any leader policy $\phi: {\cal S} \to {\cal A}$ and the follower policy $\psi: {\cal S} \to {\cal B}$, define
    \[M_{[\phi,\psi]} := \left[ {\begin{array}{*{20}c}
            e_{\phi(1)}^T \otimes e_{\psi(1)}^T \otimes e_1^T \\
            \vdots \\
            e_{\phi(|{\cal S}|)}^T \otimes e_{\psi(|{\cal S}|)}^T \otimes e_{|{\cal S}|}^T
        \end{array}} \right], \]
    where $e_{\phi(s)}$, $e_{\psi(s)}$, and $e_s$ denote the standard basis vectors in $\mathbb{R}^{|{\cal A}|}$, $\mathbb{R}^{|{\cal B}|}$, and $\mathbb{R}^{|{\cal S}|}$, respectively, $\otimes$ denotes the Kronecker product, and $M_{[\phi,\psi]} \in \mathbb{R}^{|{\cal S}| \times |{\cal S}||{\cal A}||{\cal B}|}$.
\end{enumerate}
\end{definition}
\vspace{0.3cm}

Note that we can express the transition matrix applied to the selected Q-values with the aforementioned definitions as
\begin{align*}
    & P_{a,b} M_{[a_k,b_k]} Q_k \\
    =& \left[ {\begin{array}{*{20}{c}}
    \sum_{s' \in {\cal S}} P(s'|1,a,b) Q_k(s',a_k(s'),b_k(s',a_k(s'))) \\
    \vdots \\
    \sum_{s' \in {\cal S}} P(s'||{\cal S}|,a,b) Q_k(s',a_k(s'),b_k(s',a_k(s')))
\end{array}} \right].
\end{align*}

\section{STACKELBERG Q-FUNCTION INEQUALITIES AND EPSILON-RELAXATION}
In~\cite{bowling2000convergence}, it is assumed that each player treats the opponent’s policy as a worst-case response, which is natural in Nash-type formulations of stochastic games. However, this worst-case assumption can be overly restrictive since it is rooted in the minimax structure of Nash formulations. Moreover, it cannot be directly applied to the Stackelberg setting due to its inherent asymmetry. Therefore, instead of assuming mutual worst-case responses, we adopt a Stackelberg setting and relax this assumption by introducing an $\epsilon$-relaxation on the upper bound. 

First, we establish the following lemma that shows inequalities between Q-function values in the Stackelberg setting.
\vspace{0.3cm}
\begin{lemma}[Stackelberg Q-function inequalities] \label{lem1}
For any state $s$, policies $\mu^1, \mu^2$, and iteration $k \ge 0$, the following inequalities hold:
\[ Q^1_k(s,\mu^1(s),b_k(s,\mu^1(s))) \le Q^1_k(s,a_k(s),b_k(s,a_k(s))), \]
\[ Q^2_k(s,a_k(s),\mu^2(s)) \le Q^2_k(s,a_k(s),b_k(s,a_k(s))), \]
and
\[ Q^1_*(s,\mu^1(s),b_*(s,\mu^1(s))) \le Q^1_*(s,a_*(s),b_*(s,a_*(s))), \]
\[ Q^2_*(s,a_*(s),\mu^2(s)) \le Q^2_*(s,a_*(s),b_*(s,a_*(s))). \]
\end{lemma}
\vspace{0.3cm}
\begin{proof}
We first consider the inequality for $Q^1_k$. Recall that
\[ a_k(s) \in \arg\max_{a \in {\cal A}} Q^1_k \left(s,a,b_k(s,a)\right) \]
from~\eqref{pi1}. It follows that for any $a \in {\cal A}$,
\[ Q^1_k(s,a,b_k(s,a)) \le Q^1_k(s,a_k(s),b_k(s,a_k(s))). \]
Substituting $a=\mu^1(s)$ yields the desired result. The remaining inequalities follow analogously.
\end{proof}
\vspace{0.3cm}

Note that the inequalities involving $Q^i_*$ follow directly from the definition of a Stackelberg equilibrium in~\eqref{pi2_opt} and~\eqref{pi1_opt}.

Next, we revisit the assumption in~\cite{bowling2000convergence}, which was originally introduced for Nash settings under a worst-case response framework. We move beyond the mutual worst-case response assumption and instead introduce an $\epsilon$-relaxation on the upper bound within the Stackelberg framework as follows:
\vspace{0.3cm}
\begin{assumption}[$\epsilon$-relaxed best response] \label{assumption2}
For some $\epsilon > 0$, for any state $s$, policies $\mu^1, \mu^2$, and iteration $k \ge 0$, we assume that
\[ Q^1_k(s,a_k(s),b_k(s,a_k(s))) \le Q^1_k(s,a_k(s),\mu^2(s)) + \epsilon, \]
\[ Q^2_k(s,a_k(s),b_k(s,a_k(s))) \le Q^2_k(s,\mu^1(s),b_k(s,a_k(s))) + \epsilon, \]
and
\[ Q^1_*(s,a_*(s),b_*(s,a_*(s))) \le Q^1_*(s,a_*(s),\mu^2(s)) + \epsilon, \]
\[ Q^2_*(s,a_*(s),b_*(s,a_*(s))) \le Q^2_*(s,\mu^1(s),b_*(s,a_*(s))) + \epsilon. \]
\end{assumption}
\vspace{0.3cm}

To justify~\cref{assumption2}, we show that there exists a constant $\epsilon > 0$ satisfying the required condition. To this end, we establish the following lemma by adopting the proof approach in~\cite{gosavi2006boundedness}.
\vspace{0.3cm}
\begin{lemma} \label{lem2}
For every $k \ge 0$,
\[ \left\|Q^i_k\right\|_{\infty} \le \frac{1}{1-\gamma}, \quad i=1,2. \]
\end{lemma}
\vspace{0.3cm}
\begin{proof}
From~\eqref{p1 q update}, one obtains
\begin{align*}
    & |Q^i_{k+1}(s,a,b)| \\
    \le& |r^i(s,a,b)| \\
    &+ \gamma \left|\sum_{s' \in {\cal S}} P(s' \mid s,a,b) \max_{a \in {\cal A}} Q^1_k\left(s',a,b_k(s',a)\right)\right| \\
    \le& |r^i(s,a,b)| \\
    &+ \gamma \sum_{s' \in {\cal S}} P(s' \mid s,a,b) \left|\max_{a \in {\cal A}} Q^1_k\left(s',a,b_k(s',a)\right)\right| \\
    \le& 1 + \gamma \sum_{s' \in {\cal S}} P(s' \mid s,a,b) \left| Q^1_k\left(s',a_k(s'),b_k(s',a)\right)\right| \\
    \le& 1+\gamma \max_{s' \in {\cal S}} |Q^i_k(s',a_k(s'),b_k(s',a_k(s'))|,
\end{align*}
where the second inequality follows from the triangular inequality, and the third inequality follows from~\cref{assumption1}. Applying this inequality recursively yields $\|Q^i_k\|_{\infty} \le 1+\gamma+\cdots+\gamma^k \le \frac{1}{1-\gamma}$, which completes the proof.
\end{proof}
\vspace{0.3cm}

Then, we establish the existence of $\epsilon$ satisfying~\cref{assumption2}.
\vspace{0.3cm}
\begin{lemma} [$\epsilon$-existence]
For any state $s$, policies $\mu^1, \mu^2$ and iteration $k \ge 0$, there exists $\epsilon > 0$ such that~\cref{assumption2} holds.
\end{lemma}
\vspace{0.3cm}
\begin{proof}
We first consider $Q^1_k$. For any state $s$ and actions $a,b$, we have
\[ |Q^1_k(s,a,b)| \le \frac{1}{1-\gamma} \]
by utilizing ~\cref{lem2}. Then,
\begin{align*}
    & \left|Q_k^1(s,a_k(s),b_k(s,a_k(s))) - Q_k^1(s,a_k(s),\mu^2(s))\right| \\
    \le& |Q_k^1(s,a_k(s),b_k(s,a_k(s)))| + |Q_k^1(s,a_k(s),\mu^2(s))| \\
    \le& \frac{2}{1-\gamma}.
\end{align*}
Thus, choosing $\epsilon \ge \frac{2}{1-\gamma}$ ensures that the inequality for $Q^1_k$ in~\cref{assumption2} holds. The remaining inequalities can be derived in the same manner.
\end{proof}
\vspace{0.3cm}

In summary,~\cref{lem1} follows directly from the definition, while ~\cref{assumption2} is imposed as an assumption. These results will be used in the analysis of Stackelberg Q-value iteration in the next section.

\section{FINITE-TIME ANALYSIS OF STACKELBERG Q-VALUE ITERATION IN GENERAL-SUM MARKOV GAMES}
\subsection{Construction of Upper and Lower Comparison Systems}
In this section, we first analyze the problem from the perspective of the leader. To study the convergence of~\eqref{p1 q update}, we reduce the analysis to the stability of an affine switching system. Two straightforward comparison iterations are used, which are easier to analyze: the \textit{upper iteration} provides an upper bound on~\eqref{p1 q update}, while the \textit{lower iteration} provides a lower bound. Then, we transform both iterations into comparison switching systems.

Based on~\eqref{p1 q update}, the upper iteration can be represented as 
\begin{align*}
    & Q^U_{k+1}(s,a,b) - Q^1_*(s,a,b) \\
    =& \gamma \sum_{s' \in {\cal S}} P(s'|s,a,b) \left\{Q^1_k(s',a_k(s'),b_*(s,a_k(s')))\right. \\
    &- \left.Q^1_*(s',a_k(s'),b_*(s,a_k(s')))\right\} + \gamma\epsilon
\end{align*}
and the lower iteration as
\begin{align*}
    & Q^L_{k+1}(s,a,b) - Q^1_*(s,a,b) \\
    =& \gamma \sum_{s' \in {\cal S}} P(s'|s,a,b) \left\{Q^1_k(s',a_*(s'),b_k(s,a_*(s')))\right. \\
    &- \left.Q^1_*(s',a_*(s'),b_k(s,a_*(s')))\right\} - \gamma\epsilon,
\end{align*}
which are established by the following propositions:
\vspace{0.3cm}
\begin{proposition}
Assume that $Q_0^U(s,a,b) \ge Q_0(s,a,b)$ for all $(s,a,b)$. Then, $Q_k^U(s,a,b) \ge Q_k(s,a,b)$ holds for all $(s,a,b)$ and $k \ge 0$.
\end{proposition}
\vspace{0.3cm}
\begin{proof}
Suppose that the statement holds for some $k \ge 0$. Then,
\begin{align*}
    & Q^1_{k+1}(s,a,b) - Q^1_*(s,a,b) \\
    =& \gamma \sum_{s' \in {\cal S}} P(s'|s,a,b) Q^1_k(s',a_k(s'),b_k(s,a_k(s'))) \\
    &- \gamma \sum_{s' \in {\cal S}} P(s'|s,a,b) Q^1_*(s',a_*(s'),b_*(s,a_*(s'))) \\
    \le& \gamma \sum_{s' \in {\cal S}} P(s'|s,a,b) Q^1_k(s',a_k(s'),b_k(s,a_k(s'))) \\
    &- \gamma \sum_{s' \in {\cal S}} P(s'|s,a,b) Q^1_*(s',a_k(s'),b_*(s,a_k(s'))) \\
    \le& \gamma \sum_{s' \in {\cal S}} P(s'|s,a,b) Q^1_k(s',a_k(s'),b_*(s,a_k(s'))) + \gamma\epsilon \\
    &- \gamma \sum_{s' \in {\cal S}} P(s'|s,a,b) Q^1_*(s',a_k(s'),b_*(s,a_k(s'))) \\
    =& \gamma \sum_{s' \in {\cal S}} P(s'|s,a,b) \left\{Q^1_k(s',a_k(s'),b_*(s,a_k(s')))\right. \\
    &- \left.Q^1_*(s',a_k(s'),b_*(s,a_k(s')))\right\} + \gamma\epsilon \\
    \le& \gamma \sum_{s' \in {\cal S}} P(s'|s,a,b) \left\{Q^U_k(s',a_k(s'),b_*(s,a_k(s')))\right. \\
    &- \left.Q^1_*(s',a_k(s'),b_*(s,a_k(s')))\right\} + \gamma\epsilon \\
    =& Q^U_{k+1}(s,a,b) - Q^1_*(s,a,b),
\end{align*}
where the first equality follows from~\eqref{p1 q update} and the assumption that the maximizer is unique for all states, the first and second inequalities follow from~\cref{lem1,assumption2}, and the last inequality follows from the assumption $Q_k^U(s,a,b) \ge Q_k(s,a,b)$. The proof is completed by induction.
\end{proof}

\vspace{0.3cm}
\begin{proposition}
Assume that $Q_0^L(s,a,b) \le Q_0(s,a,b)$ for all $(s,a,b)$. Then, $Q_k^L(s,a,b) \le Q_k(s,a,b)$ holds for all $(s,a,b)$ and $k \ge 0$.
\end{proposition}
\vspace{0.3cm}
\begin{proof}
Suppose that the statement holds for some $k \ge 0$. Then,
\begin{align*}
    & Q^1_{k+1}(s,a,b) - Q^1_*(s,a,b) \\
    =& \gamma \sum_{s' \in {\cal S}} P(s'|s,a,b) Q^1_k(s',a_k(s'),b_k(s,a_k(s'))) \\
    &- \gamma \sum_{s' \in {\cal S}} P(s'|s,a,b) Q^1_*(s',a_*(s'),b_*(s,a_*(s'))) \\
    \ge& \gamma \sum_{s' \in {\cal S}} P(s'|s,a,b) Q^1_k(s',a_*(s'),b_k(s,a_*(s')) \\
    &- \gamma \sum_{s' \in {\cal S}} P(s'|s,a,b) Q^1_*(s',a_*(s'),b_*(s,a_*(s'))) \\
    \ge& \gamma \sum_{s' \in {\cal S}} P(s'|s,a,b) Q^1_k(s',a_*(s'),b_k(s,a_*(s')) \\ 
    &- \gamma \sum_{s' \in {\cal S}} P(s'|s,a,b) Q^1_*(s',a_*(s'),b_k(s,a_*(s')) - \gamma\epsilon \\
    =& \gamma \sum_{s' \in {\cal S}} P(s'|s,a,b) \left\{Q^1_k(s',a_*(s'),b_k(s,a_*(s'))\right. \\
    &- \left.Q^1_*(s',a_*(s'),b_k(s,a_*(s'))\right\} - \gamma\epsilon \\
    \ge& \gamma \sum_{s' \in {\cal S}} P(s'|s,a,b) \left\{Q^L_k(s',a_*(s'),b_k(s,a_*(s'))\right. \\
    &- \left.Q^1_*(s',a_*(s'),b_k(s,a_*(s'))\right\} - \gamma\epsilon \\
    =& Q^L_{k+1}(s,a,b) - Q^1_*(s,a,b)
\end{align*}
where the first equality follows from~\eqref{p1 q update} and the assumption that the maximizer is unique for all states, the first and second inequalities follow from~\cref{lem1,assumption2}, and the last inequality follows from the assumption $Q_k^L(s,a,b) \le Q_k(s,a,b)$. The proof is completed by induction.
\end{proof}
\vspace{0.3cm}

Then, using~\cref{definition1}, the upper comparison system can be expressed as
\begin{equation} \label{upper comp sys}
    Q^U_{k+1} - Q^1_* = \gamma P M_{[a_k,b_*]} \left(Q^U_k - Q^1_*\right) + \gamma\epsilon\textbf{1},
\end{equation}
where $\textbf{1}$ is the all-ones column vector. Moreover, the lower comparison system can be expressed as
\begin{equation} \label{lower comp sys}
    Q^L_{k+1} - Q^1_* = \gamma P M_{[a_*,b_k]} \left(Q^L_k - Q^1_*\right) - \gamma\epsilon\textbf{1}.
\end{equation}
Equations~\eqref{upper comp sys} and~\eqref{lower comp sys} can be interpreted as switching systems, where $M_{[a_k,b_*]}$ and $ M_{[a_*,b_k]}$ vary with $Q_k$. The term $\gamma\epsilon\textbf{1}$ acts as a constant affine term.

\subsection{Finite-Time Error Bound}
By analyzing the two comparison systems~\eqref{upper comp sys} and~\eqref{lower comp sys}, we establish the convergence of Stackelberg Q-value iteration in general-sum Markov games as follows:
\vspace{0.3cm}
\begin{theorem} \label{thm1}
For all $k \geq 0$,
\begin{equation} \label{final conv}
    \left\| Q^1_k-Q^1_* \right\|_{\infty} \le \frac{6}{1-\gamma}\gamma^k + \frac{3\epsilon}{1-\gamma}.
\end{equation}
\end{theorem}
\vspace{0.3cm}
\begin{proof}
Taking into account the norm of~\eqref{upper comp sys}, one gets
\begin{align*}
    & \left\|Q^U_{k+1} - Q^1_*\right\|_{\infty} \\
    \le& \gamma \left\|P\right\|_{\infty} \left\|M_{[a_k,b_*]}\right\|_{\infty} \left\|Q^U_k - Q^1_*\right\|_{\infty} + \gamma \epsilon \left\|\textbf{1}\right\|_{\infty} \\
    =& \gamma \left\|Q^U_k - Q^1_*\right\|_{\infty} + \gamma\epsilon.
\end{align*}
Then, unrolling the aforementioned inequality from $i=0$ to $k-1$,
\begin{align}
    \left\|Q^U_{k+1} - Q^1_*\right\|_{\infty}
    \le& \gamma^k \left\|Q^U_0 - Q^1_*\right\|_{\infty} + \epsilon \sum_{i=1}^k \gamma \nonumber \\
    \le& \gamma^k \left(\left\|Q^U_0\right\|_{\infty} + \left\|Q^1_*\right\|_{\infty}\right) + \epsilon \sum_{i=0}^{\infty} \gamma \nonumber \\
    \le& \frac{2}{1-\gamma}\gamma^k + \frac{\epsilon}{1-\gamma}, \label{finite time error bound upper}
\end{align}
where the last inequality utilizes~\cref{assumption1,lem2}. Similarly, one can also prove the same finite-time error bound for the lower comparison system. By using the relation
\begin{align*}
    & \left\| Q^1_k-Q^1_* \right\|_{\infty} \\
    =& \left\| Q^1_k-Q^L_k + Q^L_k-Q^1_* \right\|_{\infty} \\
    \le& \left\| Q^L_k-Q^1_* \right\|_{\infty} + \left\| Q^1_k-Q^L_k \right\|_{\infty} \\
    \le& \left\| Q^L_k-Q^1_* \right\|_{\infty} + \left\| Q^U_k-Q^L_k \right\|_{\infty} \\
    \le& \left\| Q^L_k-Q^1_* \right\|_{\infty} + \left\| Q^U_k-Q^1_* + Q^1_*-Q^L_k \right\|_{\infty} \\
    \le& \left\| Q^L_k-Q^1_* \right\|_{\infty} + \left\| Q^U_k-Q^1_* \right\|_{\infty} + \left\| Q^1_*-Q^L_k \right\|_{\infty} \\
    =& 2\left\| Q^L_k-Q^1_* \right\|_{\infty} + \left\| Q^U_k-Q^1_* \right\|_{\infty},
\end{align*}
the final result is obtained by combining~\eqref{finite time error bound upper} with the corresponding error bound for the lower comparison system. Here, the first and fourth inequalities follow from the triangle inequality, while the second inequality follows from the fact that $Q_k^U-Q_k^L \ge Q_k-Q_k^L \ge 0$. 
\end{proof}
\vspace{0.3cm}

By examining the right side of~\eqref{final conv}, the first term diminishes as $k$ goes to infinity with $\gamma \in [0, 1)$. Moreover, it is apparent that the second term is a constant error that can be reduced by the smaller $\epsilon$. This suggests that the error between the Q-value at the $k$-th iteration and $Q^1_*$ is confined within a bounded range. For the follower, we can also get the same finite-time error bound using a similar method.

\section{NUMERICAL EXPERIMENTS}
We consider the following two-player general-sum Markov game to validate the theoretical results. The environment consists of a single state ${\cal S} = \{1\}$, with action spaces ${\cal A} = \{1,2\}$ for the leader and ${\cal B} = \{1,2\}$ for the follower. The transition is deterministic and always remains in the same state. The discount factor is set to $\gamma=0.8$.

The reward functions for the two players are defined as
\[ r^1(s,a,b) =
\begin{cases}
0.8 & (a=1, b=1) \\
0.2 & (a=1, b=2) \\
0.5 & (a=2, b=1) \\
0.9 & (a=2, b=2)
\end{cases} \]
and
\[ r^2(s,a,b) =
\begin{cases}
0.3 & (a=1, b=1) \\
0.9 & (a=1, b=2) \\
0.8 & (a=2, b=1) \\
0.1 & (a=2, b=2)
\end{cases}. \]
Under this setup, the follower's best response to the leader’s action is given by $b_*(s,1) = 2$ and $b_*(s,2) = 1$. Based on this best response, the leader evaluates $r^1(s,1,b_*(s,1)) = 0.2$ and $r^1(s,2,b_*(s,2)) = 0.5$. Then, the leader selects action $a_*(s) = 2$, while the follower selects $b_*(s,a_*(s)) = 1$. Therefore, the unique Stackelberg policy pair is given by $(a_*,b_*) \equiv (2,1)$.

Since the environment contains a single state, $Q^i_*$ admits a closed-form expression. We can get $Q^i_*$ values as $Q^1_*(s,a,b) = r^1(s,a,b) + \gamma r^1(s,2,1) / (1-\gamma)$ and $Q^2_*(s,a,b) = r^2(s,a,b) + \gamma r^2(s,2,1) / (1-\gamma)$.

We evaluate the convergence behavior of the Q-function value by measuring the sup-norm error $\|Q^i_k - Q^i_*\|_\infty$ across iterations. The empirical error is compared against our theoretical upper bounds derived from~\cref{thm1}, where $\epsilon$ is computed from the current Q-values at each iteration to ensure that the condition in~\cref{assumption2} holds. All experiments are repeated over five different seeds, where the Q-functions are initialized randomly at the beginning of training.

\begin{figure}[!t]
    \centering
    \includegraphics[width=\columnwidth]{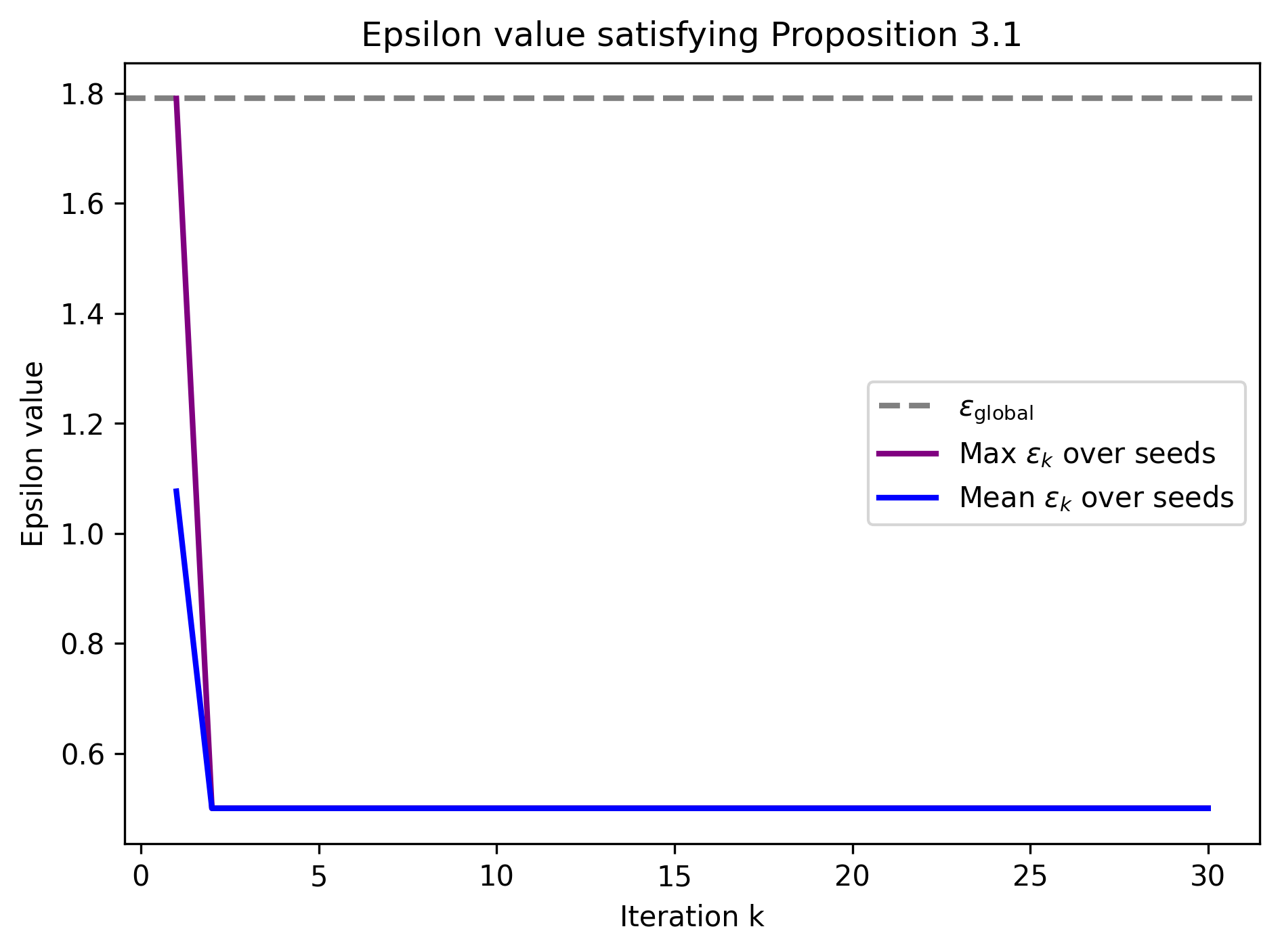}
    \caption{Illustration of epsilon values satisfying~\cref{assumption2}.}
    \label{fig1}
\end{figure}

\cref{fig1} illustrates the epsilon values required to satisfy~\cref{assumption2} across iterations. Here, we consider three types of epsilon values: the maximum value of iteration-dependent epsilon $\epsilon_k$ across different seeds (purple line), the mean value of $\epsilon_k$ across different seeds (blue line), and the minimum constant value that satisfies~\cref{assumption2} over the entire trajectory $\epsilon_{\mathrm{global}} = \max_k \epsilon_k$ (gray dashed line). We observe that the required epsilon values are large during the initial iterations. This is because the Q-functions are randomly initialized, leading to unstable policies and large value discrepancies across actions. As learning progresses, however, the required epsilon value $\epsilon_k$ decreases as the policy becomes more stable and the differences between Q-values diminish. This highlights that the global slack $\epsilon_{\mathrm{global}}$ is dominated by early iterations and can be overly conservative compared to the iteration-dependent $\epsilon_k$.

\begin{figure}[!t]
    \centering
    \includegraphics[width=\columnwidth]{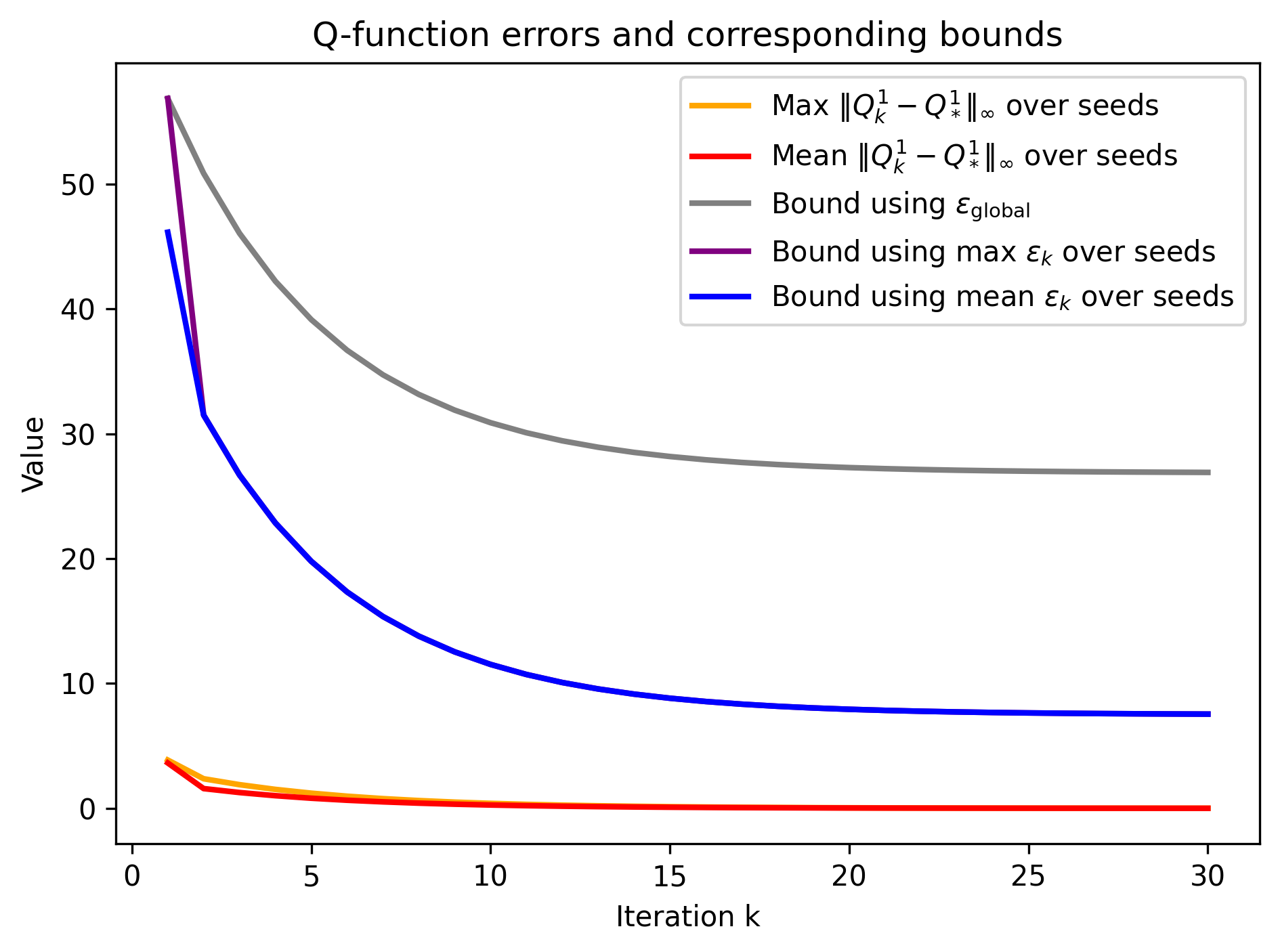}
    \caption{Illustration of Q-function error of the leader and its corresponding bounds from~\cref{thm1}.}
    \label{fig2}
\end{figure}

\begin{figure}[!t]
    \centering
    \includegraphics[width=\columnwidth]{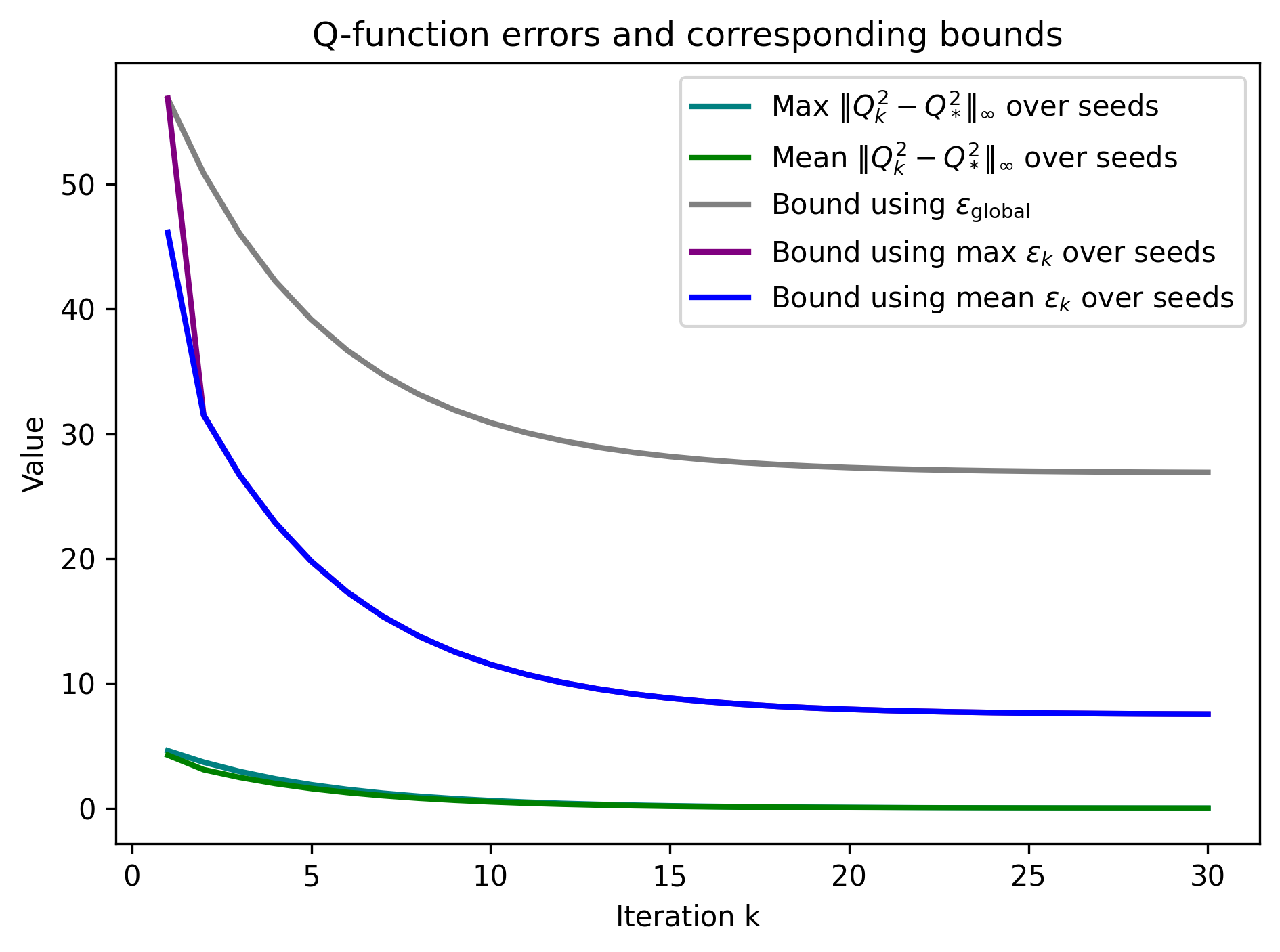}
    \caption{Illustration of Q-function error of the follower and its corresponding bounds from~\cref{thm1}.}
    \label{fig3}
\end{figure}

\cref{fig2,fig3} depict the sup-norm error of the Q-functions for both players and the corresponding theoretical upper bounds derived in~\cref{thm1} across iterations. The results show that the empirical Q-errors of both players remain consistently below the theoretical upper bounds for all iterations $k$. This confirms the validity of the bound in~\cref{thm1}. Moreover, both the empirical errors and the corresponding theoretical upper bounds decrease geometrically, which aligns with the theoretical convergence rate.

In~\cref{fig2,fig3}, the upper bound based on $\epsilon_{\mathrm{global}}$ (gray line) is significantly looser, as it depends on a worst-case value dominated by early iterations. In contrast, the adaptive upper bound using $\epsilon_k$ (blue, purple lines) provides a substantially tighter characterization of the empirical error, demonstrating that the conservativeness of the global bound primarily arises from a few initial iterations. One thing to note is that none of the upper bounds converges to zero, since all epsilon values required to satisfy~\cref{assumption2} are strictly positive. This leads to a non-vanishing residual term in the upper bound.

\section{CONCLUSIONS}
In this paper, we investigated the convergence behavior of Stackelberg Q-value iteration in two-player general-sum Markov games. We proposed a novel perspective that focuses on the dynamics of Q-functions and introduced a relaxed policy condition tailored to the asymmetric nature of Stackelberg interactions. By modeling the learning process as a switching system, we developed a systematic analytical framework that captures the time-varying and Q-function-dependent nature of Stackelberg learning dynamics. Through the construction of upper and lower comparison systems, we established explicit finite-time error bounds for the Q-functions. Our results bridge a gap in the theoretical understanding of Stackelberg learning and extend switching-system-based analysis beyond single-agent and zero-sum settings. Promising directions for future work include extending the framework to Stackelberg Q-learning with stochastic approximation and relaxing the current assumptions to obtain tighter error bounds and more general convergence guarantees.

\bibliographystyle{IEEEtran}
\bibliography{main}

@article{watkins1992q,
  title={Q-learning},
  author={Watkins, Christopher JCH and Dayan, Peter},
  journal={Machine learning},
  volume={8},
  number={3},
  pages={279--292},
  year={1992},
  publisher={Springer}
}

@inproceedings{tan1993multi,
  title={Multi-agent reinforcement learning: Independent vs. cooperative agents},
  author={Tan, Ming and others},
  booktitle={Proceedings of the tenth international conference on machine learning},
  pages={330--337},
  year={1993}
}

@incollection{littman1994markov,
  title={Markov games as a framework for multi-agent reinforcement learning},
  author={Littman, Michael L},
  booktitle={Machine learning proceedings 1994},
  pages={157--163},
  year={1994},
  publisher={Elsevier}
}

@article{hu2003nash,
  title={Nash Q-learning for general-sum stochastic games},
  author={Hu, Junling and Wellman, Michael P},
  journal={Journal of machine learning research},
  volume={4},
  number={Nov},
  pages={1039--1069},
  year={2003}
}

@inproceedings{greenwald2003correlated,
  title={Correlated Q-learning},
  author={Greenwald, Amy and Hall, Keith and Serrano, Roberto and others},
  booktitle={ICML},
  volume={3},
  pages={242--249},
  year={2003}
}

@inproceedings{cisneros2023finite,
  title={Finite-sample guarantees for nash Q-learning with linear function approximation},
  author={Cisneros-Velarde, Pedro and Koyejo, Sanmi},
  booktitle={Uncertainty in Artificial Intelligence},
  pages={424--432},
  year={2023},
  organization={PMLR}
}

@book{liberzon2003switching,
  title={Switching in systems and control},
  author={Liberzon, Daniel},
  volume={190},
  year={2003},
  publisher={Springer}
}

@article{gosavi2006boundedness,
  title={Boundedness of iterates in {Q}-learning},
  author={Gosavi, Abhijit},
  journal={Systems \& Control letters},
  volume={55},
  number={4},
  pages={347--349},
  year={2006}
}

@article{shapley1953stochastic,
  title={Stochastic games},
  author={Shapley, Lloyd S},
  journal={Proceedings of the national academy of sciences},
  volume={39},
  number={10},
  pages={1095--1100},
  year={1953}
}

@inproceedings{bowling2000convergence,
  title={Convergence problems of general-sum multiagent reinforcement learning},
  author={Bowling, Michael},
  booktitle={ICML},
  pages={89--94},
  year={2000}
}

@misc{khalil2002nonlinear,
  title={Nonlinear systems third edition (2002)},
  author={Khalil, Hassan K},
  year={2002},
  publisher={Prentice Hall}
}

@article{littman2001value,
  title={Value-function reinforcement learning in Markov games},
  author={Littman, Michael L},
  journal={Cognitive systems research},
  volume={2},
  number={1},
  pages={55--66},
  year={2001},
  publisher={Elsevier}
}

@inproceedings{fan2020theoretical,
  title={A theoretical analysis of deep Q-learning},
  author={Fan, Jianqing and Wang, Zhaoran and Xie, Yuchen and Yang, Zhuoran},
  booktitle={Learning for dynamics and control},
  pages={486--489},
  year={2020},
  organization={PMLR}
}

@article{song2021can,
  title={When can we learn general-sum Markov games with a large number of players sample-efficiently?},
  author={Song, Ziang and Mei, Song and Bai, Yu},
  journal={arXiv preprint arXiv:2110.04184},
  year={2021}
}

@inproceedings{erez2023regret,
  title={Regret minimization and convergence to equilibria in general-sum markov games},
  author={Erez, Liad and Lancewicki, Tal and Sherman, Uri and Koren, Tomer and Mansour, Yishay},
  booktitle={International Conference on Machine Learning},
  pages={9343--9373},
  year={2023},
  organization={PMLR}
}

@article{zinkevich2005cyclic,
  title={Cyclic equilibria in Markov games},
  author={Zinkevich, Martin and Greenwald, Amy and Littman, Michael},
  journal={Advances in neural information processing systems},
  volume={18},
  year={2005}
}

@article{giannou2022convergence,
  title={On the convergence of policy gradient methods to Nash equilibria in general stochastic games},
  author={Giannou, Angeliki and Lotidis, Kyriakos and Mertikopoulos, Panayotis and Vlatakis-Gkaragkounis, Emmanouil-Vasileios},
  journal={Advances in Neural Information Processing Systems},
  volume={35},
  pages={7128--7141},
  year={2022}
}

@article{li2020sample,
  title={Sample complexity of asynchronous {Q}-learning: sharper analysis and variance reduction},
  author={Li, Gen and Wei, Yuting and Chi, Yuejie and Gu, Yuantao and Chen, Yuxin},
  journal={arXiv preprint arXiv:2006.03041},
  year={2020}
}

@article{chen2021lyapunov,
  title={A {L}yapunov theory for finite-sample guarantees of asynchronous {Q}-learning and {TD}-learning variants},
  author={Chen, Zaiwei and Maguluri, Siva Theja and Shakkottai, Sanjay and Shanmugam, Karthikeyan},
  journal={arXiv preprint arXiv:2102.01567},
  year={2021}
}

@inproceedings{lee2020periodic,
  title={Periodic {Q}-learning},
  author={Lee, Donghwan and He, Niao},
  booktitle={Learning for dynamics and control},
  pages={582--598},
  year={2020}
}

@article{fiez2019convergence,
  title={Convergence of learning dynamics in stackelberg games},
  author={Fiez, Tanner and Chasnov, Benjamin and Ratliff, Lillian J},
  journal={arXiv preprint arXiv:1906.01217},
  year={2019}
}

@inproceedings{lee2020unified,
  title={A unified switching system perspective and convergence analysis of {Q}-learning algorithms},
  author={Lee, Donghwan and He, Niao},
  booktitle={34th Conference on Neural Information Processing Systems, NeurIPS 2020},
  year={2020}
}

@article{lee2021discrete,
  title={A discrete-time switching system analysis of {Q}-learning},
  author={Lee, Donghwan and Hu, Jianghai and He, Niao},
  journal={SIAM Journal on Control and Optimization (accepted)},
  year={2022}
}

@inproceedings{jeong2024finite,
  title={Finite-time error analysis of soft q-learning: Switching system approach},
  author={Jeong, Narim and Lee, Donghwan},
  booktitle={2024 IEEE 63rd Conference on Decision and Control (CDC)},
  pages={3897--3904},
  year={2024},
  organization={IEEE}
}

@inproceedings{jeong2025finite,
  title={Finite-time analysis of minimax Q-learning},
  author={Jeong, Narim and Lee, Donghwan},
  booktitle={Reinforcement Learning Conference},
  year={2025}
}

@article{he2025learning,
  title={Learning in Stackelberg Markov Games},
  author={He, Jun and Liu, Andrew L and Chen, Yihsu},
  journal={arXiv preprint arXiv:2509.16296},
  year={2025}
}

@inproceedings{sadigh2016planning,
  title={Planning for autonomous cars that leverage effects on human actions.},
  author={Sadigh, Dorsa and Sastry, Shankar and Seshia, Sanjit A and Dragan, Anca D},
  booktitle={Robotics: Science and systems},
  volume={2},
  pages={1--9},
  year={2016},
  organization={Ann Arbor, MI, USA}
}

@article{wang2021game,
  title={Game-theoretic planning for self-driving cars in multivehicle competitive scenarios},
  author={Wang, Mingyu and Wang, Zijian and Talbot, John and Gerdes, J Christian and Schwager, Mac},
  journal={IEEE Transactions on Robotics},
  volume={37},
  number={4},
  pages={1313--1325},
  year={2021},
  publisher={IEEE}
}

@book{krishna2009auction,
  title={Auction theory},
  author={Krishna, Vijay},
  year={2009},
  publisher={Academic press}
}

@inproceedings{sinha2018stackelberg,
  title={Stackelberg security games: Looking beyond a decade of success},
  author={Sinha, Arunesh and Fang, Fei and An, Bo and Kiekintveld, Christopher and Tambe, Milind},
  year={2018},
  organization={IJCAI}
}

@book{bacsar1998dynamic,
  title={Dynamic noncooperative game theory},
  author={Ba{\c{s}}ar, Tamer and Olsder, Geert Jan},
  year={1998},
  publisher={SIAM}
}

@inproceedings{fiez2020implicit,
  title={Implicit learning dynamics in stackelberg games: Equilibria characterization, convergence analysis, and empirical study},
  author={Fiez, Tanner and Chasnov, Benjamin and Ratliff, Lillian},
  booktitle={International conference on machine learning},
  pages={3133--3144},
  year={2020},
  organization={PMLR}
}

@inproceedings{vu2022stackelberg,
  title={Stackelberg policy gradient: Evaluating the performance of leaders and followers},
  author={Vu, Quoc-Liem and Alumbaugh, Zane and Ching, Ryan and Ding, Quanchen and Mahajan, Arnav and Chasnov, Benjamin and Burden, Sam and Ratliff, Lillian J},
  booktitle={ICLR 2022 Workshop on Gamification and Multiagent Solutions},
  year={2022}
}

@article{beck2023survey,
  title={A survey on bilevel optimization under uncertainty},
  author={Beck, Yasmine and Ljubi{\'c}, Ivana and Schmidt, Martin},
  journal={European Journal of Operational Research},
  volume={311},
  number={2},
  pages={401--426},
  year={2023},
  publisher={Elsevier}
}

@inproceedings{zhu2023multi,
  title={A multi-agent Q-learning with value function approximation based on single-leader multi-followers stackelberg game},
  author={Zhu, Chunxi and Yu, Wenwu and Wang, He},
  booktitle={2023 IEEE 13th International Conference on CYBER Technology in Automation, Control, and Intelligent Systems (CYBER)},
  pages={1229--1234},
  year={2023},
  organization={IEEE}
}

@inproceedings{zheng2022stackelberg,
  title={Stackelberg actor-critic: Game-theoretic reinforcement learning algorithms},
  author={Zheng, Liyuan and Fiez, Tanner and Alumbaugh, Zane and Chasnov, Benjamin and Ratliff, Lillian J},
  booktitle={Proceedings of the AAAI conference on artificial intelligence},
  volume={36},
  number={8},
  pages={9217--9224},
  year={2022}
}

@article{zhong2021can,
  title={Can reinforcement learning find Stackelberg-Nash equilibria in general-sum Markov games with myopic followers?},
  author={Zhong, Han and Yang, Zhuoran and Wang, Zhaoran and Jordan, Michael I},
  journal={arXiv preprint arXiv:2112.13521},
  year={2021}
}

@article{yu2022learning,
  title={Learning correlated stackelberg equilibrium in general-sum multi-leader-single-follower games},
  author={Yu, Yaolong and Xu, Haifeng and Chen, Haipeng},
  journal={arXiv preprint arXiv:2210.12470},
  year={2022}
}

@article{niu2025finding,
  title={Finding a Multiple Follower Stackelberg Equilibrium: A Fully First-Order Method},
  author={Niu, April and Wang, Kai and Ziani, Juba},
  journal={arXiv preprint arXiv:2509.08161},
  year={2025}
}

@article{bai2021sample,
  title={Sample-efficient learning of stackelberg equilibria in general-sum games},
  author={Bai, Yu and Jin, Chi and Wang, Huan and Xiong, Caiming},
  journal={Advances in Neural Information Processing Systems},
  volume={34},
  pages={25799--25811},
  year={2021}
}


\end{document}